\documentclass{article}

\pdfoutput=1

\usepackage{arxiv}

\usepackage{amsfonts}
\usepackage{mathtools}
\usepackage[mathscr]{euscript}
\usepackage{amsthm}
\usepackage{amsmath}
\usepackage{times}
\usepackage{graphicx} 
\usepackage{subfigure} 

\usepackage{algorithm}
\usepackage{algorithmic}
\newtheorem{theorem}{Theorem}

\newtheorem{lemma}{Lemma}
\newenvironment{sproof}{%
  \proof}{\endproof}
\newcommand{\Tr}{\operatornamewithlimits{Tr}}
\newcommand{\diag}{\operatornamewithlimits{diag}}

\newcommand{\argmin}{\operatornamewithlimits{argmin}}
\usepackage{tikz}
\newcommand*\circled[1]{\tikz[baseline=(char.base)]{
            \node[shape=circle,draw,inner sep=2pt] (char) {#1};}}

\title{Online Adaptive Principal Component Analysis and Its extensions}

\author{
  Jianjun Yuan\\
  Department of Electrical and Computer Engineering\\
  University of Minnesota\\
  Minneapolis, MN, 55455 \\
  \texttt{yuanx270@umn.edu} \\
  \And
  Andrew Lamperski \\
  Department of Electrical and Computer Engineering\\
  University of Minnesota\\
  Minneapolis, MN, 55455 \\
  \texttt{alampers@umn.edu} \\
}

\begin{document}

\maketitle

\begin{abstract} 
We propose algorithms for online principal component analysis (PCA)
and variance minimization for adaptive settings.
Previous literature has focused on upper bounding the static adversarial regret,
whose comparator is the optimal fixed action in hindsight.
However, static regret is not an appropriate metric when the underlying environment is changing.
Instead, we adopt the adaptive regret metric from the previous literature 
and propose online adaptive algorithms for PCA and variance minimization, 
that have sub-linear adaptive regret guarantees.
We demonstrate both theoretically and experimentally that
the proposed algorithms can adapt to the changing environments.

\end{abstract}

\section{Introduction}

In the general formulation of online learning,
at each time step,
the decision maker makes decision without knowing its outcome,
and suffers a loss based on the decision and the observed
outcome. Loss functions are chosen from a fixed class, but the
sequence of losses can be generated deterministically, stochastically, or adversarially.

Online learning is a very popular framework with many variants and applications, 
such as online convex optimization \cite{zinkevich2003online,shalev2012online},
online convex optimization for cumulative constraints \cite{yuan2018online},
online non-convex optimization \cite{hazan2017efficient,gao2018online}, 
online auctions \cite{blum2004online}, 
online controller design \cite{yuan2017online}, 
and online classification and regression \cite{crammer2006online}.
Additionally, recent advances in linear dynamical system identification \cite{hazan2018spectral}
and reinforcement learning \cite{fazel2018global} have been developed based on the ideas from online learning.

The standard performance metric for online learning measures  
the difference between the decision maker's cumulative loss 
and the cumulative loss of the best fixed decision in hindsight \cite{cesa2006prediction}.
We call this metric \emph{static regret}, since the comparator is the
best fixed optimum in hindsight. 
However, when the underlying environment is changing, 
due to the fixed comparator \cite{herbster1998tracking},
static regret is no longer appropriate.

Alternatively, to capture the changes of the underlying environment,
\cite{hazan2009efficient} introduced the metric called adaptive regret, 
which is defined as the maximum static regret over any contiguous time interval.

In this paper, we are mainly concerned with the problem of online
Principal Component Analysis (online PCA) for adaptive settings.
Previous online PCA algorithms are based on either online gradient descent 
or matrix exponentiated gradient algorithms
\cite{tsuda2005matrix,warmuth2006online,warmuth2008randomized,niew2016onlinepca}.
These works bound the static regret for online PCA algorithms, but do
not address adaptive regret. 
As argued above, static regret is not appropriate under changing environments.

This paper gives an efficient algorithm for online PCA and variance
minimization in changing environments. The proposed method mixes
the randomized algorithm from \cite{warmuth2008randomized} with a
fixed-share step \cite{herbster1998tracking}.
This is inspired by the work of \cite{cesa2012new,cesa2012mirror},
which shows that the Hedge algorithm \cite{freund1997decision} together
with a fixed-share step provides low regret under a variety of
measures, including adaptive regret.

Furthermore, we extend the idea of the additional fixed-share step 
to the online adaptive variance minimization 
in two different parameter spaces: the space of unit vectors and the
simplex.
In the Section~\ref{sec:exp} on experiments\footnote{code available at https://github.com/yuanx270/online-adaptive-PCA},
we also test our algorithm's effectiveness. In particular, we show
that our proposed algorithm can adapt to the changing environment
faster than the previous online PCA algorithm.

While it is  possible to apply the algorithm in
\cite{hazan2009efficient} to solve the online adaptive PCA and
variance minimization problems 
with the similar order of the adaptive regret as in this paper,
it requires running a pool of algorithms in parallel. Compared to our algorithm,
Running this pool algorithms requires complex implementation
that increases the running time per step by a factor of $\log T$.

\subsection{Notation}

Vectors are denoted by bold lower-case symbols. The $i$-th element of
a vector $\mathbf{q}$ is denoted by $q_i$.
The $i$-th element of a sequence of vectors at time step $t$, $\mathbf{x_t}$, 
is denoted by $x_{t,i}$.

For two probability vectors $\mathbf{q}, \mathbf{w} \in \mathbb{R}^n$, we use $d(\mathbf{q},\mathbf{w})$
to represent the relative entropy between them, which is defined as 
$\sum_{i=1}^n q_i\ln(\frac{q_i}{w_i})$. 
The $\ell_1$-norm and $\ell_2$-norm of the vector $\mathbf{q}$ 
are denoted as $\left\|\mathbf{q}\right\|_1$, $\left\|\mathbf{q}\right\|_2$, respectively.
$\mathbf{q_{1:T}}$ is the sequence of vectors $\mathbf{q_1},\dots,\mathbf{q_T}$, 
and $m(\mathbf{q_{1:T}})$ is defined to be equal to $\sum\limits_{t=1}^{T-1}D_{TV}(\mathbf{q_{t+1}},\mathbf{q_t})$,
where $D_{TV}(\mathbf{q_t},\mathbf{q_{t-1}})$ is defined as
$\sum\limits_{i:q_{t,i}\ge q_{t-1,i}} (q_{t,i}-q_{t-1,i})$. The
expected value operator is denoted by $\mathbb{E}$. 

When we refer to a matrix, we use capital letters such as $P$ and $Q$ with $\left\|Q\right\|_2$ representing the spectral norm.
For the identity matrix, we use $I$.
The quantum relative entropy between two 
density matrices\footnote{A density matrix is a symmetric positive
  semi-definite matrix with trace equal to 1. 
Thus, the eigenvalues of a density matrix form a probability vector.} 
$P$ and $Q$ is defined as $\Delta(P,Q) = \Tr(P\ln P)- \Tr(P\ln Q)$, 
where $\ln P$ is the matrix logarithm for symmetric positive definite matrix $P$
(and $\exp(P)$ is the matrix exponential).

\section{Problem Formulation}

The goal of the PCA (uncentered) algorithm is to find a rank $k$
projection matrix $P$ 
that minimizes the compression loss: $\sum\limits_{t=1}^T
\left\|\mathbf{x_t}-P\mathbf{x_t}\right\|_2^2$. In this case,
 $P \in \mathbb{R}^{n\times n}$ 
 must be a symmetric positive semi-definite matrix
with only $k$ non-zero eigenvalues which are all equal to 1.



In online PCA,
the data points come in a stream. 
At each time $t$, the algorithm first chooses a projection matrix
$P_t$ with rank $k$, then the data point $\mathbf{x_t}$ is revealed,
and a compression loss of
$\left\|\mathbf{x_t}-P_t\mathbf{x_t}\right\|_2^2$ is incurred.

The online PCA algorithm \cite{warmuth2008randomized}
aims to minimize the static regret $\mathcal{R}_s$ ,which is the difference 
between the total expected compression loss
and the loss of the best projection matrix $P^*$ chosen in hindsight:
\begin{equation}
\label{eq::static_regret_pca}
\mathcal{R}_s = \sum\limits_{t=1}^T \mathbb{E}[\Tr((I-P_t)\mathbf{x_t}\mathbf{x_t}^T)] - \sum\limits_{t=1}^T \Tr((I-P^*)\mathbf{x_t}\mathbf{x_t}^T).
\end{equation}
The algorithm from \cite{warmuth2008randomized} is randomized and the
expectation is taken over the distribution of $P_t$ matrices. The
matrix 
$P^*$ is the solution to the following optimization problem with
$\mathcal{S}$ being the set of rank-$k$ projection matrices:
\begin{equation}
\label{eq::best_fixed_sol_PCA}
\min_{P\in \mathcal{S}} \sum\limits_{t=1}^T \Tr((I-P)\mathbf{x_t}\mathbf{x_t}^T)
\end{equation}

Algorithms that minimize static regret will converge to $P^*$, which
is the best projection for the entire data set. However, in many 
scenarios the data generating process changes over time. 
In this case,
a solution that adapts to changes in the data set may be
desirable. To model environmental variation, 
several notions of dynamically varying regret have been
proposed \cite{herbster1998tracking,hazan2009efficient,cesa2012new}. 
In this paper, we study adaptive regret $\mathcal{R}_a$ from \cite{hazan2009efficient}, 
which results in the following online adaptive PCA problem:
\begin{equation}
\label{eq::adaptive_regret_pca}
\begin{array}{l}
\mathcal{R}_a = \max\limits_{[r,s]\subset [1,T]}\Big\{
\sum\limits_{t=r}^s \mathbb{E}[\Tr((I-P_t)\mathbf{x_t}\mathbf{x_t}^T)]\\
\quad\quad\quad\quad\quad\quad\quad\quad -
\min\limits_{U\in \mathcal{S}}\sum\limits_{t=r}^s \Tr((I-U)\mathbf{x_t}\mathbf{x_t}^T)  \Big\}
\end{array}
\end{equation}
In the next few sections,  
we will present an algorithm that achieves low adaptive regret.

\section{Learning the Adaptive Best Subset of Experts}

\begin{algorithm}[tb]
    \caption{Adaptive Best Subset of Experts}
    \label{alg::alg1}
\begin{algorithmic}[1]
    \STATE {\bfseries Input:} $1\le k < n$ and an initial probability vector $\mathbf{w_1} \in \mathcal{B}_{\text{n-k}}^\text{n}$.
    \FOR{$t=1$ {\bfseries to} $T$}
    \STATE Use Algorithm \ref{alg::alg2} with input $d=n-k$ to decompose $\mathbf{w_t}$ into $\sum_j p_j\mathbf{r_j}$, which is a
    convex combination of at most $n$ corners of $\mathbf{r_j}$.
    \STATE Randomly select a corner $\mathbf{r}=\mathbf{r_j}$ with associated probability $p_j$.
    \STATE Use the k components with zero entries in the drawn corner $\mathbf{r}$ as the selected subset of experts.
    \STATE Receive loss vector $\mathbf{\ell_t}$.
    \STATE Update $\mathbf{w_{t+1}}$ as:
    \begin{subequations}
    \label{eq::our_expert_update}
    \begin{align}
    \label{eq::v_t+1}
    &v_{t+1,i} = \frac{w_{t,i}\exp(-\eta\ell_{t,i})}{\sum_{j=1}^n \exp(-\eta\ell_{t,j})}\\
    \label{eq::fix_share_expert}
    &\hat{w}_{t+1,i} = \frac{\alpha}{n} + (1-\alpha)v_{t+1,i} \\
    \label{eq::w_t+1}
    &\mathbf{w_{t+1}} = \text{cap}_{\text{n-k}}(\mathbf{\hat{w}_{t+1}})
    \end{align}
    \end{subequations}
    where $\text{cap}_{\text{n-k}}()$ calls Algorithm \ref{alg::alg3}.
    \ENDFOR
\end{algorithmic}
\end{algorithm}

\begin{algorithm}[tb]
    \caption{Mixture Decomposition \cite{warmuth2008randomized}}
    \label{alg::alg2}
\begin{algorithmic}[1]
    \STATE {\bfseries Input:} $1\le d < n$ and $\mathbf{w}\in \mathcal{B}_\text{d}^\text{n}$.
    \REPEAT 
    \STATE Let $\mathbf{r}$ be a corner for a subset of $d$ non-zero components of $\mathbf{w}$ 
    that includes all components of $\mathbf{w}$ equal to $\frac{\left|\mathbf{w}\right|}{d}$.
    \STATE Let $s$ be the smallest of the $d$ chosen components of $\mathbf{r}$ and $l$ be the largest value
    of the remaining $n-d$ components.
    \STATE update $\mathbf{w}$ as $\mathbf{w}-\min(ds,\left|\mathbf{w}\right|-dl)\mathbf{r}$ and {\bfseries Output} $p$ and $\mathbf{r}$.
    \UNTIL{$\mathbf{w}=0$}
\end{algorithmic}
\end{algorithm}

\begin{algorithm}[tb]
    \caption{Capping Algorithm \cite{warmuth2008randomized}}
    \label{alg::alg3}
\begin{algorithmic}[1]
    \STATE {\bfseries Input:} probability vector $\mathbf{w}$ and set size $d$.
    \STATE Let $\mathbf{w}^{\downarrow}$ index the vector in decreasing order, that is, $\mathbf{w_1}^{\downarrow} = \max(\mathbf{w})$.
    \IF{$\max(\mathbf{w})\le 1/d$}
      \STATE {\bfseries return} $\mathbf{w}$.
    \ENDIF
    \STATE $i=1$.
    \REPEAT 
    \STATE (* Set first $i$ largest components to $1/d$ and normalize the rest to $(d-i)/d$ *)
    \STATE $\mathbf{\tilde{w}} = \mathbf{w}$, $\tilde{w}_j^{\downarrow} = 1/d$, for $j = 1,\dots,i$.
    \STATE $\tilde{w}_j^{\downarrow} = \frac{d-i}{d}\frac{\tilde{w}_j^{\downarrow}}{\sum_{l=j}^n \tilde{w}_l^{\downarrow}}$, for $j = i+1,\dots,n$.
    \STATE $i = i + 1$.
    \UNTIL{$\max(\mathbf{\tilde{w}})\le 1/d$}.
\end{algorithmic}
\end{algorithm}

In \cite{warmuth2008randomized}
it was shown that online PCA can be
viewed as an extension of a simpler problem known as the \emph{best subset of experts} problem. In
particular, they first propose an online algorithm to solve the best
subset of experts problem, and then they show how to modify the
algorithm to solve PCA problems. In this section, we show how the
addition of a fixed-share step \cite{herbster1998tracking,cesa2012new}
can lead to an algorithm for an adaptive variant of the best subset of
experts problem. Then we will show how to extend the resulting
algorithm to PCA problems. 

The \emph{adaptive best subset of experts}  problem can be described as follows:
we have $n$ experts making decisions at each time $t$.
Before revealing the loss vector $\mathbf{\ell_t}\in \mathbb{R}^n$ associated with the experts' decisions at time $t$,
we select a subset of experts of size $n-k$ (represented by vector $\mathbf{v_t}$) to try to minimize the adaptive regret defined as:
\begin{equation}
\label{eq::adaptive_regret_expert}
\mathcal{R}_a^{\text{subexp}} = \max_{[r,s]\subset [1,T]}\Big\{
\sum\limits_{t=r}^s \mathbb{E}[\mathbf{v_t}^T\mathbf{\ell_t}] -
\min_{\mathbf{u}\in \mathcal{S}_{\text{vec}}}\sum\limits_{t=r}^s \mathbf{u}^T\mathbf{\ell_t}  \Big\}.
\end{equation}
Here, the expectation is taken over the probability distribution of $\mathbf{v_t}$.
Both $\mathbf{v_t}$ and $\mathbf{u}$ are in $\mathcal{S}_{\text{vec}}$ which
denotes the vector set 
with only $n-k$ non-zero elements equal to 1.

Similar to the static regret case from  \cite{warmuth2008randomized},
the problem in Eq.(\ref{eq::adaptive_regret_expert}) is equivalent to:
\scriptsize
\begin{equation}
\label{eq::reform_adaptive_regret_expert}
\mathcal{R}_a^{\text{subexp}} = \max_{[r,s]\subset [1,T]}\Big\{
\sum\limits_{t=r}^s (n-k)\mathbf{w_t}^T\mathbf{\ell_t} -
\min_{\mathbf{q}\in\mathcal{B}_{\text{n-k}}^\text{n} }\sum\limits_{t=r}^s (n-k)\mathbf{q}^T\mathbf{\ell_t}  \Big\}
\end{equation}
\normalsize
where $\mathbf{w_t}\in\mathcal{B}_{\text{n-k}}^\text{n}$, and $\mathcal{B}_{\text{n-k}}^\text{n}$
represents the capped probability simplex
defined as $\sum_{i=1}^nw_{t,i} = 1$ and $0\le w_{t,i}\le 1/(n-k)$, $\forall i$.

Such equivalence is due to the Theorem 2 in \cite{warmuth2008randomized} ensuring that
any vector $\mathbf{q}\in\mathcal{B}_{\text{n-k}}^\text{n}$ can be decomposed as convex combination of 
at most $n$ corners of $\mathbf{r_j}$ by using Algorithm \ref{alg::alg2}, 
where the corner $\mathbf{r_j}$ is defined 
as having $n-k$ non-zero elements equal to $1/(n-k)$. 
As a result, the corner can be sampled by the associated probability obtained from the convex combination,
which is a valid subset selection vector $\mathbf{v_t}$ with the multiplication of $n-k$.

\textbf{Connection to the online adaptive PCA.} The problem from
Eq.(\ref{eq::adaptive_regret_expert}) can be viewed as restricted
version of 
the online adaptive PCA problem from
Eq.(\ref{eq::adaptive_regret_pca}).
In particular, say that $I-P_t = \mathrm{diag}(\mathbf{v_t})$. This corresponds
to restricting $P_t$ to be diagonal. If 
$\mathbf{\ell_t}$ is the diagonal of $\mathbf{x_t}\mathbf{x_t}^T$,
then the objectives of Eq.(\ref{eq::adaptive_regret_expert}) and  
Eq.(\ref{eq::adaptive_regret_pca}) are equal.

We now return to the adaptive best subset of experts problem.
When $r=1$ and $s=T$, the problem reduces to the standard static regret minimization problem,
which is studied in \cite{warmuth2008randomized}. 
Their solution applies  the basic Hedge Algorithm to obtain a probability distribution for the experts,
and modifies the distribution to select a subset of the experts.

To deal with the adaptive regret considered in Eq.(\ref{eq::reform_adaptive_regret_expert}),
we propose the Algorithm \ref{alg::alg1},
%
which is a simple modification to Algorithm 1 in \cite{warmuth2008randomized}.
More specifically, we add Eq.(\ref{eq::fix_share_expert}) when updating $\mathbf{w_{t+1}}$ in Step $7$,
%
which is called a \emph{fixed-share} step.
This is inspired by the analysis in \cite{cesa2012new}, 
which shows that the online adaptive best expert problem can be solved 
by simply adding this fixed-share step to the standard Hedge algorithm.

With the Algorithm \ref{alg::alg1},
the following lemma can be obtained:

\begin{lemma}
\label{lem::adaptive_expert_step_ineq}
For all $t\ge 1$, all $\mathbf{\ell_t} \in [0,1]^n$, and for all $\mathbf{q_t}\in \mathcal{B}_{\text{n-k}}^\text{n}$, 
Algorithm \ref{alg::alg1} satisfies 
\begin{equation*}
\mathbf{w_t}^T\mathbf{\ell_t}(1-\exp(-\eta)) - \eta \mathbf{q_t}^T\mathbf{\ell_t} \le
\sum_{i=1}^n q_{t,i}\ln(\frac{v_{t+1,i}}{\hat{w}_{t,i}})
\end{equation*}
\end{lemma}

\begin{proof}

With the update in Eq.(\ref{eq::our_expert_update}),
for any $\mathbf{q_t} \in\mathcal{B}_{\text{n-k}}^\text{n}$, we have
\footnotesize
\begin{equation}
d(\mathbf{q_t},\mathbf{w_t})-d(\mathbf{q_t},\mathbf{v_{t+1}}) = -\eta \mathbf{q_t}^T\mathbf{\ell_t}-\ln(\sum_{j=1}^n w_{t,j}\exp(-\eta\ell_{t,j}))
\end{equation}
\normalsize

Also, from the proof of Theorem 1 in \cite{warmuth2008randomized}, we have 
$-\ln(\sum_{j=1}^n w_{t,j}\exp(-\eta\ell_{t,j})) \ge \mathbf{w_t}^T\mathbf{\ell_t}(1-\exp(-\eta))$. 
Thus, we will get 
\small
\begin{equation}
\label{eq::online_pca_thm1_ineq}
d(\mathbf{q_t},\mathbf{w_t})-d(\mathbf{q_t},\mathbf{v_{t+1}}) \ge -\eta \mathbf{q_t}^T\mathbf{\ell_t} + \mathbf{w_t}^T\mathbf{\ell_t}(1-\exp(-\eta))
\end{equation}
\normalsize

Moreover, Eq.(\ref{eq::w_t+1}) is the solution to the following projection problem as shown in \cite{warmuth2008randomized}:
\begin{equation}
\mathbf{w_t} = \argmin\limits_{\mathbf{w}\in\mathcal{B}_{\text{n-k}}^\text{n}} d(\mathbf{w},\mathbf{\hat{w}_t})
\end{equation}
Since the relative entropy is one kind of Bregman divergence \cite{bregman1967relaxation,censor1981iterative},
the Generalized Pythagorean Theorem holds \cite{herbster2001tracking}:
\begin{equation}
\label{eq::general_pythagorean}
d(\mathbf{q_t},\mathbf{\hat{w}_t}) - d(\mathbf{q_t},\mathbf{w_t}) \ge d(\mathbf{w_t},\mathbf{\hat{w}_t})\ge 0
\end{equation}
where the last inequality is due to the non-negativity of Bregman divergence.

Combining Eq.(\ref{eq::online_pca_thm1_ineq}) with Eq.(\ref{eq::general_pythagorean})
and expanding the left part of $d(\mathbf{q_t},\mathbf{\hat{w}_t})-d(\mathbf{q_t},\mathbf{v_{t+1}})$, we arrive at Lemma \ref{lem::adaptive_expert_step_ineq}.
\end{proof}


Now we are ready to state the following theorem 
to upper bound the adaptive regret $\mathcal{R}_a^{\text{subexp}}$:

\begin{theorem}
\label{thm::adaptive_subset_expert}
{\it If we run the Algorithm \ref{alg::alg1} to select a subset of $n-k$ experts,
then for any sequence of loss vectors $\mathbf{\ell_1}$, $\dots$, $\mathbf{\ell_T}$ $\in$ $[0,1]^n$ with $T\ge 1$,
$\min_{\mathbf{q}\in\mathcal{B}_{\text{n-k}}^\text{n} }\sum\limits_{t=r}^s (n-k)\mathbf{q}^T\mathbf{\ell_t} \le L$,
$\alpha = 1/(T(n-k)+1)$, 
$D = (n-k)\ln(n(1+(n-k)T))+1$,
and $\eta = \ln(1+\sqrt{2D/L})$, we have 
\begin{equation*}
\mathcal{R}_a^{\text{subexp}} \le O(\sqrt{2LD}+D) 
\end{equation*}
}
\end{theorem}

\begin{sproof}
After showing the inequality from Lemma
\ref{lem::adaptive_expert_step_ineq}, the main work that remains is to
sum the right side from $t=1$ to $T$ and provide an upper bound.
This is achieved by following the proof of the Proposition 2 in \cite{cesa2012new}.
The main idea is to expand the term $\sum_{i=1}^n q_{t,i}\ln(\frac{v_{t+1,i}}{\hat{w}_{t,i}})$ as follows:
\scriptsize
\begin{equation}
\label{eq::analysis_entropy_expert}
\begin{array}{ll}
\sum_{i=1}^n q_{t,i}\ln(\frac{v_{t+1,i}}{\hat{w}_{t,i}})
=& \underbrace{\sum_{i=1}^n\Big(q_{t,i}\ln\frac{1}{\hat{w}_{t,i}} - q_{t-1,i}\ln\frac{1}{v_{t,i}}\Big)}_A\\
& + \underbrace{\sum_{i=1}^n\Big(q_{t-1,i}\ln\frac{1}{v_{t,i}} - q_{t,i}\ln\frac{1}{v_{t+1,i}}\Big)}_B
\end{array}
\end{equation}
\normalsize

Then we can upper bound the expression of $A$ with the \emph{fixed-share} step,
since $\hat{w}_{t,i}$ is lower bounded by $\frac{\alpha}{n}$.
We can telescope the expression of $B$.
Then our desired upper bound can be obtained with the help of Lemma 4 from \cite{freund1997decision}.
\end{sproof}

For space purposes, all the detailed proofs for the omitted/sketched
proofs are in the appendix.

\section{Online Adaptive PCA}

Recall that the online adaptive PCA problem is below:
\begin{equation}
\label{eq::adaptive_regret_pca_alg_sec}
\begin{array}{l}
\mathcal{R}_a = \max\limits_{[r,s]\subset [1,T]}\Big\{
\sum\limits_{t=r}^s \mathbb{E}[\Tr((I-P_t)\mathbf{x_t}\mathbf{x_t}^T)] \\
\quad\quad\quad\quad\quad\quad\quad\quad-
\min\limits_{U\in \mathcal{S}}\sum\limits_{t=r}^s \Tr((I-U)\mathbf{x_t}\mathbf{x_t}^T)  \Big\}
\end{array}
\end{equation}
where $\mathcal{S}$ is the rank $k$ projection matrix set.


Again, inspired by \cite{warmuth2008randomized}, 
we first reformulate the above problem into the following 'capped probability simplex' form:
\begin{equation}
\label{eq::ref_adaptive_pca}
\begin{array}{ll}
\mathcal{R}_a =&
 \max\limits_{[r,s]\subset [1,T]}\Big\{ \sum\limits_{t=r}^s (n-k)\Tr(W_t\mathbf{x_t}\mathbf{x_t}^T) \\
 &\quad\quad\quad\quad - \min\limits_{Q \in \mathscr{B}_{\text{n-k}}^\text{n}}\sum\limits_{t=r}^s (n-k)\Tr(Q\mathbf{x_t}\mathbf{x_t}^T) \Big\}
\end{array}
\end{equation}
where $W_t \in \mathscr{B}_{\text{n-k}}^\text{n}$, 
and $\mathscr{B}_{\text{n-k}}^\text{n}$ is the set of all density
matrices with eigenvalues bounded by $1/(n-k)$. Note that
$\mathscr{B}_{\text{n-k}}^\text{n}$ can be expressed as the convex set
$\{W: W\succeq 0, \left\|W\right\|_2 \le 1/(n-k), \Tr(W) = 1\}$. 

\begin{algorithm}[tb]
    \caption{Uncentered online adaptive PCA}
    \label{alg::adaptive_pca}
\begin{algorithmic}[1]
    \STATE {\bfseries Input:} $1\le k < n$ and an initial density matrix $W_1 \in \mathscr{B}_{n-k}^n$.
    \FOR{$t=1$ {\bfseries to} $T$}
    \STATE Apply eigendecomposition to $W_t$ as $W_t = \bar{D}\diag(\mathbf{w_t})\bar{D}^T$.
    \STATE Apply Algorithm \ref{alg::alg2} with $d=n-k$ to the vector $\mathbf{w_t}$ to decompose it into 
    a convex combination $\sum_j p_j\mathbf{r_j}$ of at most $n$ corners $\mathbf{r_j}$.
    \STATE Randomly select a corner $\mathbf{r}=\mathbf{r_j}$ with the associated probability $p_j$.
    \STATE Form a density matrix $R = (n-k)\bar{D}\diag(\mathbf{r})\bar{D}^T$
    \STATE Form a rank $k$ projection matrix $P_t = I - R$
    \STATE Obtain the data point $\mathbf{x_t}$, which incurs the compression loss $\left\|\mathbf{x_t}-P_t\mathbf{x_t}\right\|_2^2$ 
    and expected compression loss $(n-k)\Tr(W_t\mathbf{x_t}\mathbf{x_t}^T)$.
    \STATE Update $W_{t+1}$ as:
    \scriptsize
    \begin{subequations}
    \label{eq::our_pca_update}
    \begin{align}
    \label{eq::V_t+1}
    &V_{t+1} = \frac{\exp(\ln W_t - \eta \mathbf{x_t}\mathbf{x_t}^T)}{\Tr(\exp(\ln W_t - \eta \mathbf{x_t}\mathbf{x_t}^T))}\\
    \label{eq::fix_share_pca}
    &\hat{w}_{t+1,i} = \frac{\alpha}{n} + (1-\alpha)v_{t+1,i}, 
    \widehat{W}_{t+1} = U\diag(\mathbf{\hat{w}_{t+1}})U^T \\
    \label{eq::W_t+1}
    &W_{t+1} = \text{cap}_{n-k}(\widehat{W}_{t+1})
    \end{align}
    \end{subequations}
    \normalsize
    where we apply eigendecomposition to $V_{t+1}$ as $V_{t+1} = U\diag(\mathbf{v_{t+1}})U^T$,
    and $\text{cap}_{n-k}(W)$ invokes Algorithm \ref{alg::alg3} with input being the eigenvalues of $W$.
    \ENDFOR
\end{algorithmic}
\end{algorithm}

The static regret online PCA is a special case of the above problem with $r = 1$ and $s = T$,
and is solved by Algorithm 5 in \cite{warmuth2008randomized}.

Follow the idea in the last section, 
we propose the Algorithm \ref{alg::adaptive_pca}.
Compared with the Algorithm 5 in \cite{warmuth2008randomized},
we have added the fixed-share step in the update of $W_{t+1}$ at step $9$,
which will be shown to be the key in upper bounding the adaptive regret of the online PCA.

In order to analyze Algorithm \ref{alg::adaptive_pca}, we need a few
supporting results. The first result comes from \cite{warmuth2006online}:
\begin{theorem}\cite{warmuth2006online}
\label{thm::quantum_ineq}
{\it
For any sequence of data points $\mathbf{x_1}$, $\dots$, $\mathbf{x_T}$ 
with $\mathbf{x_t}\mathbf{x_t}^T \preceq I$ and for any learning rate $\eta$,
the following bound holds for any matrix $Q_t \in \mathscr{B}_{\text{n-k}}^\text{n}$ 
with the update in Eq.(\ref{eq::V_t+1}):
\small
\begin{equation*}
\Tr(W_t\mathbf{x_t}\mathbf{x_t}^T) \le \frac{\Delta(Q_t,W_t) - \Delta(Q_t,V_{t+1}) 
+ \eta\Tr(Q_t\mathbf{x_t}\mathbf{x_t}^T)}{1-\exp(-\eta)}
\end{equation*}
    
}
\end{theorem}

Based on the above theorem's result, we have the following lemma:

\begin{lemma}
\label{lem::adaptive_pca_step_ineq}
For all $t\ge 1$, all $\mathbf{x_t}$ with $\left\|\mathbf{x_t}\right\|_2 \le 1$,
and for all $Q_t\in\mathscr{B}_{n-k}^n$, Algorithm \ref{alg::adaptive_pca} satisfies:
\begin{equation}
\begin{array}{l}
\Tr(W_t\mathbf{x_t}\mathbf{x_t}^T)(1-\exp(-\eta)) -\eta\Tr(Q_t\mathbf{x_t}\mathbf{x_t}^T) \\
\quad \quad \quad \quad \le -\Tr(Q_t\ln \widehat{W}_t) + \Tr(Q_t\ln V_{t+1})
\end{array}
\end{equation}
\end{lemma}

\begin{proof}

First, we need to 
reformulate the above inequality in Theorem \ref{thm::quantum_ineq}, we have:
\begin{equation}
\label{eq::V_update_ineq}
\begin{array}{l}
\Delta(Q_t,W_t) - \Delta(Q_t,V_{t+1}) \\
\quad\quad\ge -\eta\Tr(Q_t\mathbf{x_t}\mathbf{x_t}^T) + \Tr(W_t\mathbf{x_t}\mathbf{x_t}^T)(1-\exp(-\eta))
\end{array}
\end{equation}
which is very similar to the Eq.(\ref{eq::online_pca_thm1_ineq}).

As is shown in \cite{warmuth2008randomized}, the Eq.(\ref{eq::W_t+1}) is the solution to
the following optimization problem:
\begin{equation}
W_t = \argmin\limits_{W\in\mathscr{B}_{\text{n-k}}^\text{n}}\Delta(W,\widehat{W}_t)
\end{equation}

As a result, the Generalized Pythagorean Theorem holds \cite{herbster2001tracking} for any $Q_t\in\mathscr{B}_{\text{n-k}}^\text{n}$:
\begin{equation}
\Delta(Q_t,\widehat{W}_t) - \Delta(Q_t,W_t) \ge \Delta(W_t,\widehat{W}_t) \ge 0
\end{equation}

Combining the above inequality with Eq.(\ref{eq::V_update_ineq}) and expanding the left part, we have
\begin{equation}
\begin{array}{l}
\Tr(W_t\mathbf{x_t}\mathbf{x_t}^T)(1-\exp(-\eta)) -\eta\Tr(Q_t\mathbf{x_t}\mathbf{x_t}^T) \\
\quad \quad \quad \quad \le -\Tr(Q_t\ln \widehat{W}_t) + \Tr(Q_t\ln V_{t+1})
\end{array}
\end{equation}
which proves the result.
\end{proof}

In the next theorem, we show that 
with the addition of the fixed-share step in Eq.(\ref{eq::fix_share_pca}),
we can solve the online adaptive PCA problem in Eq.(\ref{eq::adaptive_regret_pca_alg_sec}).

\begin{theorem}
\label{thm::adaptive_pca}
{\it
For any sequence of data points $\mathbf{x_1}$, $\dots$, $\mathbf{x_T}$
with $\left\|\mathbf{x_t}\right\|_2 \le 1$, and for 
$\min_{Q\in\mathscr{B}_{\text{n-k}}^\text{n}}\sum\limits_{t=r}^s (n-k)\Tr(Q\mathbf{x_t}\mathbf{x_t}^T) \le L$,
if we run Algorithm \ref{alg::adaptive_pca} with
$\alpha = 1/(T(n-k)+1)$, 
$D = (n-k)\ln(n(1+(n-k)T))+1$,
and $\eta = \ln(1+\sqrt{2D/L})$, for any $T\ge 1$ we have:
\begin{equation*}
\mathcal{R}_a \le O(\sqrt{2LD}+D) 
\end{equation*}    
}
\end{theorem}

\begin{sproof}
The proof idea is the same as in the proof of Theorem \ref{thm::adaptive_subset_expert}.
After getting the inequality relationship in Lemma \ref{lem::adaptive_pca_step_ineq}
which has a similar form as in Lemma \ref{lem::adaptive_expert_step_ineq},
we need to upper bound sum over $t$ of the right side. 
To achieve this, we first reformulate it as two parts below:
\begin{equation}
\label{eq::quantum_split_two_parts}
\begin{array}{l}
-\Tr(Q_t\ln \widehat{W}_t) + \Tr(Q_t\ln V_{t+1}) = \bar{A} + \bar{B}
\end{array}
\end{equation}
where $\bar{A} = -\Tr(Q_t\ln \widehat{W}_t) + \Tr(Q_{t-1}\ln V_t)$, 
and $\bar{B} = - \Tr(Q_{t-1}\ln V_t) + \Tr(Q_t\ln V_{t+1})$.

The first part can be upper bounded with the help of the fixed-share step
in lower bounding the singular value of $\hat{w}_{t,i}$.
After telescoping the second part, we can get the desired upper bound 
with the help of Lemma 4 from \cite{freund1997decision}.
\end{sproof}

\section{Extension to Online Adaptive Variance Minimization}

In this section, we study the closely related problem of online
adaptive variance minimization. The problem is defined as follows:
At each time $t$, we first select a vector $\mathbf{y_t}\in \Omega$,
and then a covariance matrix $C_t\in \mathbb{R}^{n\times n}$ such that
$0 \preceq C_t\preceq I$ is revealed. The goal is to
minimize the adaptive regret defined as:
\begin{equation}
\label{eq::general_var}
\mathcal{R}_{a}^{\text{var}} = \max_{[r,s]\subset [1,T]}\Big\{
\sum\limits_{t=r}^s \mathbb{E}[\mathbf{y_t}^TC_t\mathbf{y_t}] -
\min_{\mathbf{u}\in\Omega}\sum\limits_{t=r}^s \mathbf{u}^TC_t\mathbf{u}  \Big\}
\end{equation}
where the expectation is taken over the probability distribution of $\mathbf{y_t}$.

This problem has two different situations corresponding to different parameter space $\Omega$ of $\mathbf{y_t}$ and $\mathbf{u}$.

\textbf{Situation 1:} 
When $\Omega$ is the set of $\{\mathbf{x} | \left\|\mathbf{x}\right\|_2 = 1 \}$ (e.g., the unit vector space),
the solution to $\min_{\mathbf{u}\in\Omega}\sum_{t=r}^s \mathbf{u}^TC_t\mathbf{u}$ is the minimum eigenvector
of the matrix $\sum_{t=r}^sC_t$.

\textbf{Situation 2:} When $\Omega$ is the probability simplex (e.g., $\Omega$ is equal to $\mathcal{B}_{\text{1}}^\text{n}$),
it corresponds to the risk minimization in stock portfolios \cite{markowitz1952portfolio}. 

We will start with \textbf{Situation 1} since it is highly related to the previous section.

\subsection{Online Adaptive Variance Minimization over the Unit vector space}

We begin with the observation of the following equivalence \cite{warmuth2006online}:
\begin{equation}
\min_{\left\|\mathbf{u}\right\|_2 = 1} \mathbf{u}^TC\mathbf{u} = \min_{U\in \mathscr{B}_{\text{1}}^\text{n}}\Tr(UC) 
\end{equation}
where $C$ is any covariance matrix, 
and $\mathscr{B}_{\text{1}}^\text{n}$ is the set of all density matrices.

Thus, the problem in (\ref{eq::general_var}) can be reformulated as:
\begin{equation}
\label{eq::var_unit_form}
\mathcal{R}_{a}^{\text{var-unit}} = \max_{[r,s]\subset [1,T]}\Big\{
\sum\limits_{t=r}^s \Tr(Y_tC_t) -
\min_{U\in\mathscr{B}_{\text{1}}^\text{n}}\sum\limits_{t=r}^s \Tr(UC_t)  \Big\}
\end{equation}
where $Y_t\in\mathscr{B}_{\text{1}}^\text{n}$.

To see the equivalence between $\mathbb{E}[\mathbf{y_t}^TC_t\mathbf{y_t}]$ in Eq.(\ref{eq::general_var})
and $\Tr(Y_tC_t)$, 
we do the eigendecomposition of $Y_t = \sum_{i=1}^n\sigma_i\mathbf{y_i}\mathbf{y_i}^T$.
Then $\Tr(Y_tC_t)$ is equal to $\sum_{i=1}^n\sigma_i\Tr(\mathbf{y_i}\mathbf{y_i}^TC_t)$
$=$ $\sum_{i=1}^n\sigma_i\mathbf{y_i}^TC_t\mathbf{y_i}$. 
Since $Y_t\in\mathscr{B}_{\text{1}}^\text{n}$, 
the vector $\mathbf{\sigma}$ is a simplex vector, 
and $\sum_{i=1}^n\sigma_i\mathbf{y_i}^TC_t\mathbf{y_i}$ is equal to $\mathbb{E}[\mathbf{y_i}^TC_t\mathbf{y_i}]$ 
with probability distribution defined by the vector $\mathbf{\sigma}$.

If we examine Eq.(\ref{eq::var_unit_form}) and (\ref{eq::ref_adaptive_pca}) together,
we will see that they share some similarities:
First, they are almost the same if we set $n-k =1$ in Eq.(\ref{eq::ref_adaptive_pca}).
Also, $\mathbf{x_t}\mathbf{x_t}^T$ in Eq.(\ref{eq::ref_adaptive_pca}) is a special case of $C_t$ in Eq.(\ref{eq::var_unit_form}). 

Thus, it is possible to apply Algorithm \ref{alg::adaptive_pca} 
to solving the problem (\ref{eq::var_unit_form}) by setting $n-k =
1$. In this case, Algorithms \ref{alg::alg2} and \ref{alg::alg3} are
not needed.
This is summarized in Algorithm \ref{alg::adaptive_var_unit}.

\begin{algorithm}[tb]
    \caption{Online adaptive variance minimization over unit sphere}
    \label{alg::adaptive_var_unit}
\begin{algorithmic}[1]
    \STATE {\bfseries Input:} an initial density matrix $Y_1 \in \mathscr{B}_{1}^n$.
    \FOR{$t=1$ {\bfseries to} $T$}
    \STATE Perform eigendecomposition $Y_t = \widehat{D}\diag(\sigma_t)\widehat{D}^T$.
    \STATE Use the vector $\mathbf{y_t}=\widehat{D}[:,j]$ with probability $\sigma_{t,j}$.
    \STATE Receive covariance matrix $C_t$, which incurs the loss $\mathbf{y_t}^TC_t\mathbf{y_t}$ and expected loss $\Tr(Y_tC_t)$.
    \STATE Update $Y_{t+1}$ as:
    \begin{subequations}
    \small
    \label{eq::var_unit_update}
    \begin{align}
    \label{eq::var_unit_v_t+1}
    &V_{t+1} = \frac{\exp(\ln Y_t - \eta C_t)}{\Tr(\exp(\ln Y_t - \eta C_t))}\\
    \label{eq::var_unit_fix_share_pca}
    &\sigma_{t+1,i} = \frac{\alpha}{n} + (1-\alpha)v_{t+1,i}, 
    Y_{t+1} = \widehat{U}\diag(\sigma_{t+1})\widehat{U}^T
    \end{align}
    \end{subequations}
    where we apply eigendecomposition to $V_{t+1}$ as $V_{t+1} = \widehat{U}\diag(\mathbf{v_{t+1}})\widehat{U}^T$.
    \ENDFOR
\end{algorithmic}
\end{algorithm}

The theorem below is analogous to Theorem \ref{thm::adaptive_pca} in
the case that $n-k = 1$.

\begin{theorem}
\label{thm::adaptive_var_unit}
{\it
For any sequence of covariance matrices $C_1$, $\dots$, $C_T$
with $0\preceq C_t \preceq I$, and for 
$\min_{U\in\mathscr{B}_{\text{1}}^\text{n}}\sum\limits_{t=r}^s \Tr(UC_t) \le L$,
if we run Algorithm \ref{alg::adaptive_var_unit} with
$\alpha = 1/(T+1)$, 
$D = \ln(n(1+T))+1$,
and $\eta = \ln(1+\sqrt{2D/L})$, for any $T\ge 1$ we have:
\begin{equation*}
\mathcal{R}_a^{\text{var-unit}} \le O(\sqrt{2LD}+D) 
\end{equation*}    
}
\end{theorem}

\begin{sproof}
Similar inequality can be obtained as in Lemma \ref{lem::adaptive_pca_step_ineq}
by using the result of Theorem 2 in \cite{warmuth2006online}.
The rest follows the proof of Theorem \ref{thm::adaptive_pca}.
\end{sproof}

In order to apply the above theorem, we need to either estimate the step size $\eta$ heuristically
or estimate the upper bound $L$,
which may not be easily done.

In the next theorem, we show that 
we can still upper bound the $\mathcal{R}_a^{\text{var-unit}}$ without knowing $L$,
but the upper bound is a function of time horizon $T$ instead of the upper bound $L$.

Before we get to the theorem, we need the following lemma which lifts the
vector case of Lemma 1 in \cite{cesa2012new} to the density matrix case:
\begin{lemma}
\label{lem::mat_lem1_bianchi}
For any $\eta \ge 0$, $t\ge 1$, any covariance matrix $C_t$ with $0\preceq C_t\preceq I$,
and for any $Q_t\in\mathscr{B}_{1}^n$, Algorithm \ref{alg::adaptive_var_unit} satisfies:
\begin{equation*}
\begin{array}{l}
\Tr(Y_tC_t) - \Tr(Q_tC_t) \\
\quad \quad \quad \quad \quad \le \frac{1}{\eta}\Big( \Tr(Q_t\ln V_{t+1}) - \Tr(Q_t\ln Y_t) \Big) + \frac{\eta}{2}
\end{array}
\end{equation*}
\end{lemma}

Now we are ready to present the upper bound on the regret for
Algorithm~\ref{alg::adaptive_var_unit}. 
\begin{theorem}
\label{thm::T_depend_adaptive_var_unit}
{\it
For any sequence of covariance matrices $C_1$, $\dots$, $C_T$
with $0\preceq C_t \preceq I$,
if we run Algorithm \ref{alg::adaptive_var_unit} with
$\alpha = 1/(T+1)$
and $\eta = \frac{\sqrt{\ln(n(1+T))}}{\sqrt{T}}$, for any $T\ge 1$ we have:
\begin{equation*}
\mathcal{R}_a^{\text{var-unit}} \le O\Big(\sqrt{T\ln\big(n(1+T)\big)}\Big) 
\end{equation*}    
}
\end{theorem}

\begin{proof}

In the proof, we will use two cases of $Q_t$: $Q_t\in\mathscr{B}_{1}^n$, and $Q_t = 0$.

From Lemma \ref{lem::mat_lem1_bianchi}, the following inequality is valid for both cases of $Q_t$:
\begin{equation}
\begin{array}{l}
\Tr(Y_tC_t) - \Tr(Q_tC_t) \\
\quad \quad \quad \quad \quad \le \frac{1}{\eta}\Big( \Tr(Q_t\ln V_{t+1}) - \Tr(Q_t\ln Y_t) \Big) + \frac{\eta}{2}
\end{array}
\end{equation}

Follow the same analysis as in the proof of Theorem \ref{thm::adaptive_pca},
we first do the eigendecomposition to $Q_t$ as $Q_t = \widetilde{D}\diag(q_t)\widetilde{D}^T$.
Since $\left\|q_t\right\|_1$ is either $1$ or $0$, we will re-write the above inequality as:
\begin{equation}
\begin{array}{l}
\left\|q_t\right\|_1\Tr(Y_tC_t) - \Tr(Q_tC_t) \\
\quad \quad \quad \le \frac{1}{\eta}\Big( \Tr(Q_t\ln V_{t+1}) - \Tr(Q_t\ln Y_t) \Big) + \frac{\eta}{2}\left\|q_t\right\|_1
\end{array}
\end{equation}

Analyzing the term $\Tr(Q_t\ln V_{t+1}) - \Tr(Q_t\ln Y_t)$ in the above inequality 
is the same as the analysis of the Eq.(\ref{eq::quantum_split_two_parts}) in the appendix.

Thus, summing over $t=1$ to $T$ to the above inequality, and setting $Q_t = Q\in\mathscr{B}_1^n$ for $t=r,\dots,s$
and $0$ elsewhere,
we will have 
\begin{equation}
\begin{array}{l}
\sum\limits_{t=r}^s\Tr(Y_tC_t) - \min\limits_{U\in\mathscr{B}_1^n}\sum\limits_{t=r}^s\Tr(UC_t)\\
\quad\quad\quad\quad\quad\quad\le \frac{1}{\eta}\Big(\ln\frac{n}{\alpha}+T\ln\frac{1}{1-\alpha}\Big) + \frac{\eta}{2}T,
\end{array}
\end{equation} 
since it holds for any $Q\in\mathscr{B}_1^n$.

After plugging in the expression of $\eta$ and $\alpha$, we will have
\begin{equation}
\begin{array}{l}
\sum\limits_{t=r}^s\Tr(Y_tC_t) - \min\limits_{U\in\mathscr{B}_1^n}\sum\limits_{t=r}^s\Tr(UC_t) \\
\quad\quad\quad\quad\quad\quad\quad\quad\quad\quad\le O\Big(\sqrt{T\ln\big(n(1+T)\big)}\Big) 
\end{array}
\end{equation}

Since the above inequality holds for any $1\le r\le s\le T$, 
we will put a $\max\limits_{[r,s]\subset [1,T]}$ in the left part,
which proves the result.
\end{proof}

\subsection{Online Adaptive Variance Minimization over the Simplex space}

We first re-write the problem in Eq.(\ref{eq::general_var}) when
$\Omega$ is the simplex below:
\small
\begin{equation}
\label{eq::var_simplex_form}
\mathcal{R}_{a}^{\text{var-sim}} = \max_{[r,s]\subset [1,T]}\Big\{
\sum\limits_{t=r}^s \mathbb{E}[\mathbf{y_t}^TC_t\mathbf{y_t}] -
\min_{\mathbf{u}\in\mathcal{B}_1^n}\sum\limits_{t=r}^s \mathbf{u}^TC_t\mathbf{u}  \Big\}
\end{equation}
\normalsize
where $\mathbf{y_t}\in\mathcal{B}_1^n$, and $\mathcal{B}_1^n$ is the simplex set.

When $r = 1$ and $s = T$, the problem reduces to the static regret problem,
which is solved in \cite{warmuth2006online} by the exponentiated gradient algorithm as below:
\begin{equation}
y_{t+1,i} = \frac{y_{t,i}\exp\big(-\eta(C_t\mathbf{y_t})_i\big)}{\sum_i y_{t,i}\exp\big(-\eta(C_t\mathbf{y_t})_i\big)}
\end{equation}

As is done in the previous sections, we add the fixed-share step after the above update,
which is summarized in Algorithm \ref{alg::adaptive_var_simplex}.

\begin{algorithm}[tb]
    \caption{Online adaptive variance minimization over simplex}
    \label{alg::adaptive_var_simplex}
\begin{algorithmic}[1]
    \STATE {\bfseries Input:} an initial vector $\mathbf{y_1} \in \mathcal{B}_{1}^n$.
    \FOR{$t=1$ {\bfseries to} $T$}
    \STATE Receive covariance matrix $C_t$.
    \STATE Incur the loss $\mathbf{y_t}^TC_t\mathbf{y_t}$.
    \STATE Update $\mathbf{y_{t+1}}$ as:
    \begin{subequations}
    \small
    \label{eq::var_simplex_update}
    \begin{align}
    \label{eq::var_simplex_v_t+1}
    &v_{t+1,i} = \frac{y_{t,i}\exp\big(-\eta(C_t\mathbf{y_t})_i\big)}{\sum_i y_{t,i}\exp\big(-\eta(C_t\mathbf{y_t})_i\big)},\\
    \label{eq::var_unit_fix_share_pca}
    &y_{t+1,i} = \frac{\alpha}{n} + (1-\alpha)v_{t+1,i}.
    \end{align}
    \end{subequations}
    \ENDFOR
\end{algorithmic}
\end{algorithm}


With the update of $y_t$ in the Algorithm \ref{alg::adaptive_var_simplex},
we have the following theorem:
\begin{theorem}
\label{thm::adaptive_var_simplex}
{\it
For any sequence of covariance matrices $C_1$, $\dots$, $C_T$
with $0\preceq C_t \preceq I$, and for 
$\min_{\mathbf{u}\in\mathcal{B}_{\text{1}}^\text{n}}\sum\limits_{t=r}^s \mathbf{u}^TC_t\mathbf{u} \le L$,
if we run Algorithm \ref{alg::adaptive_var_simplex} with
$\alpha = 1/(T+1)$, $c = \frac{\sqrt{2\ln\big((1+T)n\big)+2}}{\sqrt{L}}$,
$b = \frac{c}{2}$, $a = \frac{b}{2b+1}$, 
and $\eta = 2a$, for any $T\ge 1$ we have:
\begin{equation*}
\mathcal{R}_a^{\text{var-sim}} \le 2\sqrt{2L\Big(\ln\big((1+T)n\big)+1\Big)} + 2\ln\big((1+T)n\big)
\end{equation*}    
}
\end{theorem}

\section{Experiments}
\label{sec:exp}
\begin{figure}
\vskip 0.0in
  \centering
  \subfigure[]{
    \label{fig::synthetic_subspace_all} 
    \includegraphics[height=5.cm]{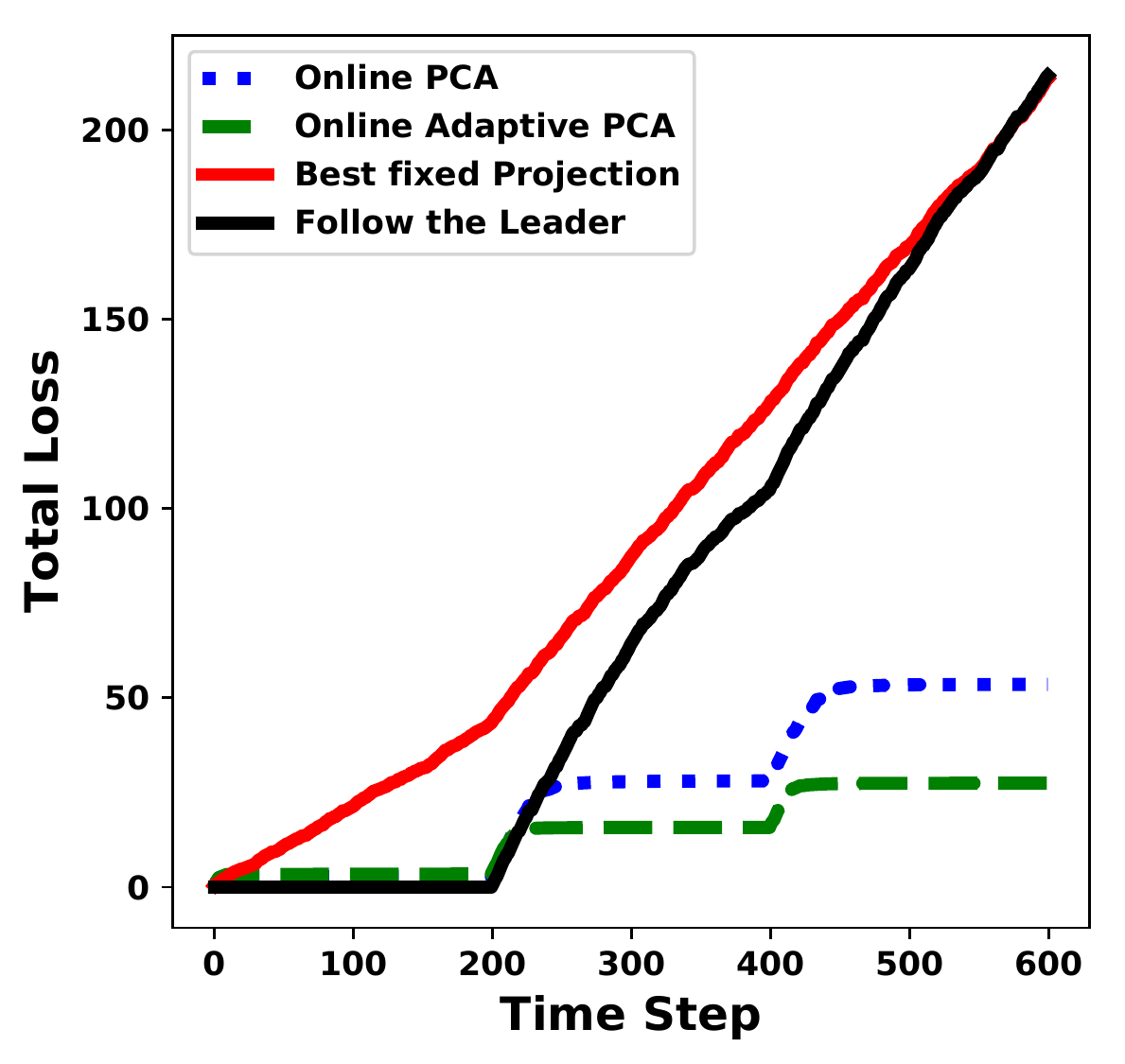}}
  \hspace{.1in}
  \subfigure[]{
    \label{fig::synthetic_subspace_online} 
    \includegraphics[height=5.cm]{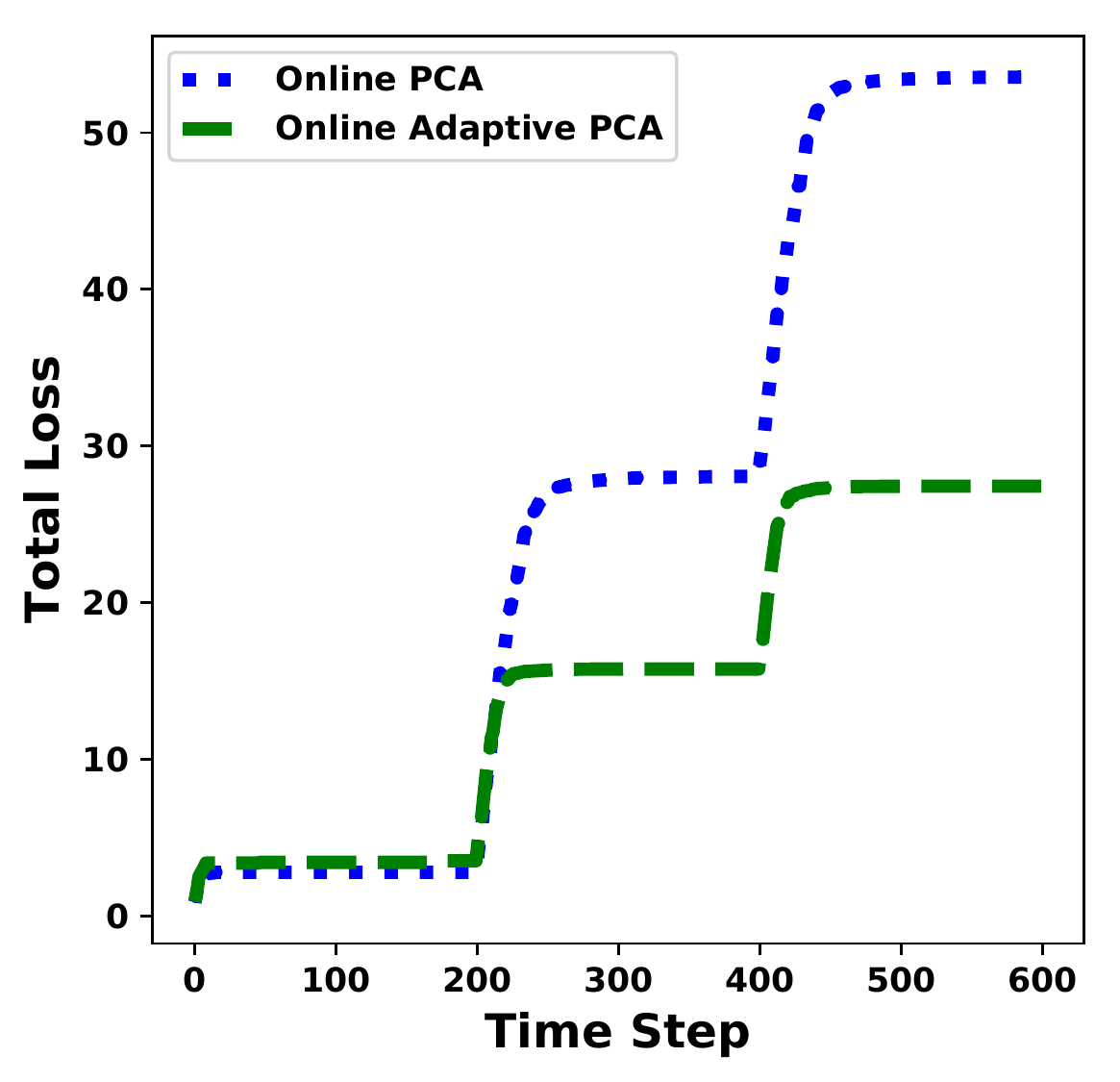}}
  \caption{
           Fig.\ref{fig::synthetic_subspace_all}: 
           The cumulative loss of the toy example with data samples coming from three different subspaces.
           Fig.\ref{fig::synthetic_subspace_online}:
           The detailed comparison for the two online algorithms. }
  \label{fig::synthetic_subspace} 
\vskip 0.in
\end{figure}

In this section, we use two examples to illustrate the effectiveness of our proposed online adaptive PCA algorithm.
The first example is synthetic, which shows that our proposed algorithm (denoted as Online Adaptive PCA) 
can adapt to the changing subspace faster than the method of
\cite{warmuth2008randomized}.
The second example uses the practical dataset Yale-B to demonstrate that the proposed algorithm
can have lower cumulative loss in practice when the data/face samples are coming from different persons.

The other algorithms that are used as comparators are:
1. Follow the Leader algorithm (denoted as Follow the Leader) \cite{kalai2005efficient}, 
which only minimizes the loss on the past history;
2. The best fixed solution in hindsight (denoted as Best fixed Projection),
which is the solution to the Problem described in Eq.(\ref{eq::best_fixed_sol_PCA});
3. The online static PCA (denoted as Online PCA) \cite{warmuth2008randomized}.
Other PCA algorithms are not included, since they are not designed for regret minimization.

\subsection{A Toy Example}

\begin{figure}
\vskip 0.0in
\centering
\includegraphics[height=6cm]{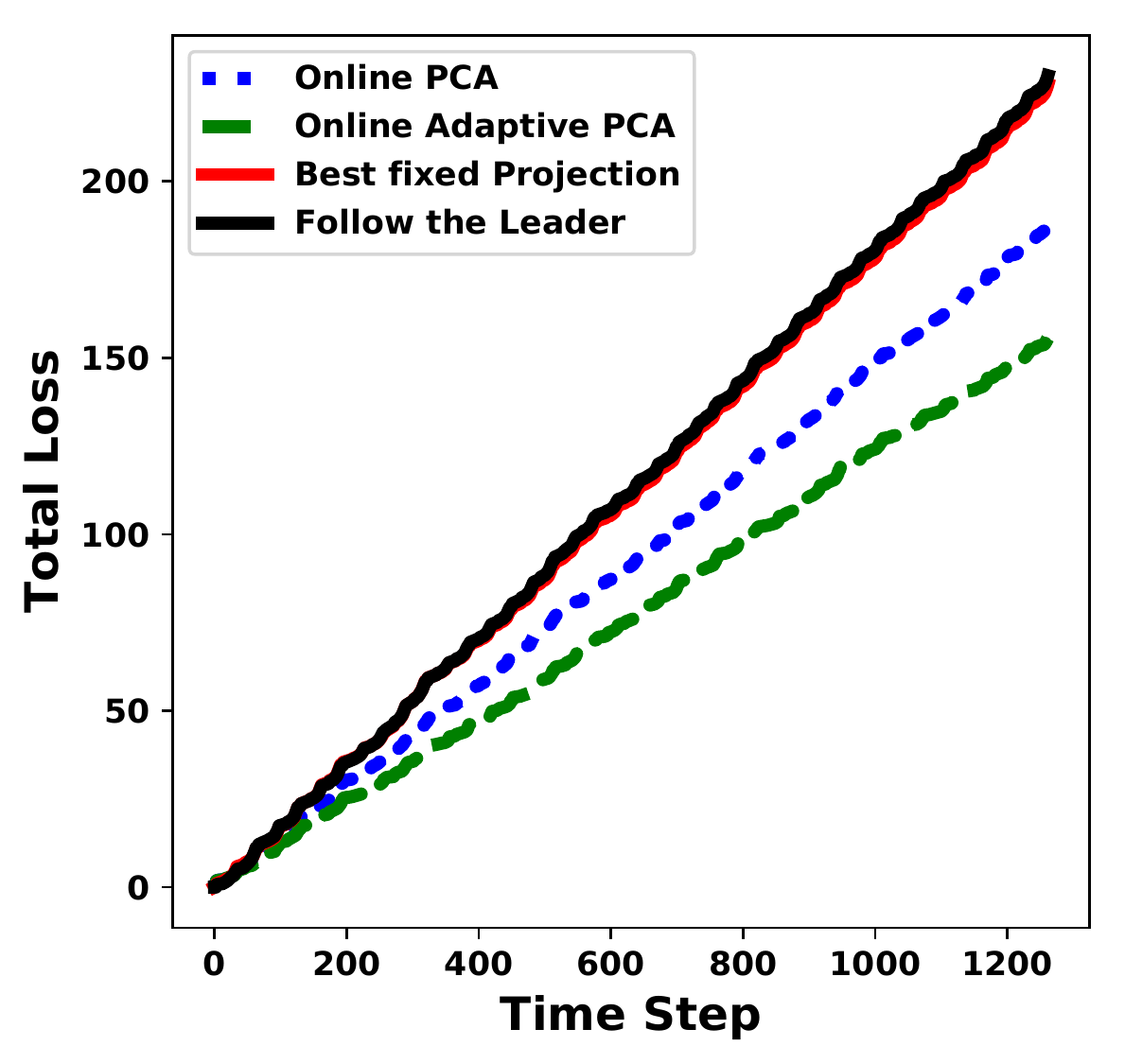}
\caption{The cumulative loss for the face example with data samples coming from 20 different persons}
\label{fig::yale_face}
\vskip 0.in
\end{figure}

In this toy example, we create the synthetic data samples coming from changing subspace/environment,
which is a similar setup as in \cite{warmuth2008randomized}.
The data samples are divided into three equal time intervals, 
and each interval has 200 data samples.
The 200 data samples within same interval is randomly generated by a Gaussian distribution 
with zero mean and data dimension equal to 20, and
the covariance matrix is randomly generated with rank equal to 2.
In this way, the data samples are from some unknown 2-dimensional subspace,
and any data sample with $\ell_2$-norm greater than 1 is normalized to 1.
Since the stepsize used in the two online algorithms is determined by 
the upper bound of the batch solution, we first find the upper bound and 
plug into the stepsize function, which gives $\eta = 0.19$.
We can tune the stepsize heuristically in practice 
and in this example we just use $\eta = 1$ and $\alpha =1\mathrm{e}{-5}$.

After all data samples are generated, we apply the previously mentioned algorithms
with $k=2$ and obtain the cumulative loss as a function of time steps,
which is shown in Fig.\ref{fig::synthetic_subspace}.
From this figure we can see that:
1. Follow the Leader algorithm is not appropriate in the setting 
where the sequential data is shifting over time.
2. The static regret is not a good metric under this setting,
since the best fixed solution in hindsight is suboptimal.
3. Compared with Static PCA, the proposed Adaptive PCA can
adapt to the changing environment faster,
which results in lower cumulative loss and
is more appropriate when the data is shifting over time.

\subsection{Face data Compression Example}
\label{sec:face}
In this example, we use the Yale-B dataset which is a collection of
face images. The data is split into 
20 time intervals corresponding to 20 different people. Within each interval,
there are 64 face image samples.
Like the previous example, we first normalize the data
to ensure its $\ell_2$-norm not greater than 1.
We use $k = 2$, which is the same as the previous example.
The stepsize $\eta$ is also tuned heuristically like the previous example,
which is equal to $5$ and $\alpha = 1\mathrm{e}{-4}$.

We apply the previously mentioned algorithms 
and again obtain the cumulative loss as the function of time steps,
which is displayed in Fig.\ref{fig::yale_face}.
From this figure we can see that
although there is no clear bumps indicating the shift from one subspace to another
as the Fig.\ref{fig::synthetic_subspace} of the toy example,
our proposed algorithm still has the lowest cumulative loss,
which indicates that upper bounding the adaptive regret
is still effective when the compressed faces are coming from different persons.

\section{Conclusion}

In this paper, we propose an online adaptive PCA algorithm, which
augments the previous online static PCA algorithm with a fixed-share
step. 
However,
different from the previous online PCA algorithm which is designed
to minimize the static regret,
the proposed online adaptive PCA algorithm aims to
minimize the adaptive regret
which is more appropriate when the underlying environment is changing
or the sequential data is shifting over time.
We demonstrate theoretically and experimentally that our algorithm can adapt to the
changing environments. 
Furthermore, we extend the online adaptive PCA algorithm 
to online adaptive variance minimization problems.

One may note that the proposed algorithms suffer from the per-iteration computation complexity of $O(n^3)$
due to the eigendecomposition step, although some tricks mentioned in \cite{arora2012stochastic} 
could be used to make it comparable with
incremental PCA of $O(k^2n)$.
For the future work,
one possible direction is to investigate algorithms with slightly worse adaptive regret bound
but with better per-iteration computation complexity.

\bibliography{online_PCA}

\bibliographystyle{unsrt}

\newpage

\begin{center}
\textbf{\large Supplementary}
\end{center}

\appendix

The supplementary material contains proofs of the main results of the
paper along with supporting results.

Before presenting the proofs, we need the following lemma from previous literature:
\begin{lemma}\cite{freund1997decision}
\label{lem::freund_ineq}
Suppose $0\le L\le \tilde{L}$ and $0 < R \le \tilde{R}$. Let $\beta = g(\tilde{L}/\tilde{R})$
where $g(z) = 1/(1+\sqrt{2/z})$. Then
\begin{equation*}
\frac{-L\ln\beta + R}{1-\beta}\le L + \sqrt{2\tilde{L}\tilde{R}} + R
\end{equation*}
\end{lemma}

Additionally, we need the following classic bound on traces for
postive semidefinite matrices. See, e.g. \cite{tsuda2005matrix}.
\begin{lemma}
\label{lem::matrix_pos_sym_ineq}
For any positive semi-definite matrix $A$ and any symmetric matrices $B$ and $C$,
$B\preceq C$ implies $\Tr(AB)\le\Tr(AC)$. 
\end{lemma}

\section{Proof of Theorem \ref{thm::adaptive_subset_expert}}

\begin{proof}

Fix $1\le r\le s \le T$. We set $\mathbf{q_t} = \mathbf{q}\in
\mathcal{B}_{n-k}^n$ for $t=r,\dots,s$ and $0$ elsewhere. Thus, we
have that  $\|\mathbf{q_t}\|_1$ is either $0$ or $1$. 

According to Lemma \ref{lem::adaptive_expert_step_ineq}, for both cases of $\mathbf{q_t}$, we have
\begin{equation}
\label{eq::initial_ineq_entropy}
\left\|\mathbf{q_t}\right\|_1 \mathbf{w_t}^T\mathbf{\ell_t}(1-\exp(-\eta)) - \eta \mathbf{q_t}^T\mathbf{\ell_t} \le
\sum_{i=1}^n q_{t,i}\ln(\frac{v_{t+1,i}}{\hat{w}_{t,i}})
\end{equation}

The analysis for $\sum_{i=1}^n
q_{t,i}\ln(\frac{v_{t+1,i}}{\hat{w}_{t,i}})$ follows the Proof of
Proposition $2$ in \cite{cesa2012new}.  
We describe the steps for completeness, since it is helpful for
understanding the effect of the fixed-share step,
Eq.(\ref{eq::fix_share_expert}). This analysis will be crucial for the
understanding how the fixed-share step can be applied to PCA
problems. 
\scriptsize
\begin{equation}
\label{eq::analysis_entropy_expert}
\begin{array}{ll}
\sum_{i=1}^n q_{t,i}\ln(\frac{v_{t+1,i}}{\hat{w}_{t,i}})
=& \underbrace{\sum_{i=1}^n\Big(q_{t,i}\ln\frac{1}{\hat{w}_{t,i}} - q_{t-1,i}\ln\frac{1}{v_{t,i}}\Big)}_A\\
& + \underbrace{\sum_{i=1}^n\Big(q_{t-1,i}\ln\frac{1}{v_{t,i}} - q_{t,i}\ln\frac{1}{v_{t+1,i}}\Big)}_B
\end{array}
\end{equation}
\normalsize

For the expression of $A$, we have
\small
\begin{equation}
\begin{array}{ll}
A =& \sum\limits_{i:q_{t,i}\ge q_{t-1,i}}\Big((q_{t,i}-q_{t-1,i})\ln\frac{1}{\hat{w}_{t,i}} +
 q_{t-1,i}\ln\frac{v_{t,i}}{\hat{w}_{t,i}}\Big)\\
 & +
 \sum\limits_{i:q_{t,i}<q_{t-1,i}}\Big(\underbrace{(q_{t,i}-q_{t-1,i})\ln\frac{1}{v_{t,i}}}_{\le 0}+q_{t,i}\ln\frac{v_{t,i}}{\hat{w}_{t,i}}\Big)
\end{array}
\end{equation}
\normalsize

Based on the update in Eq.(\ref{eq::our_expert_update}), we have $1/\hat{w}_{t,i}\le n/\alpha$
and $v_{t,i}/\hat{w}_{t,i}\le 1/(1-\alpha)$. Plugging the bounds into the above equation, we have
\small
\begin{equation}
\begin{array}{ll}
A \le & \underbrace{\sum\limits_{i:q_{t,i}\ge q_{t-1,i}} (q_{t,i}-q_{t-1,i})}_{= D_{TV}(\mathbf{q_t},\mathbf{q_{t-1}})}\ln\frac{n}{\alpha} \\
& + \underbrace{\Big(\sum\limits_{i:q_{t,i}\ge q_{t-1,i}}q_{t-1,i}+
\sum\limits_{i:q_{t,i}<q_{t-1,i}}q_{t,i}\Big)}_{=\left\|\mathbf{q_t}\right\|_1-D_{TV}(\mathbf{q_t},\mathbf{q_{t-1}})}
\ln\frac{1}{1-\alpha} .
\end{array}
\end{equation}
\normalsize

Telescoping the expression of $B$, substituting the above inequality in Eq.(\ref{eq::analysis_entropy_expert}),
and summing over $t=2,\dots,T$, we have 
\scriptsize
\begin{multline}
\sum\limits_{t=2}^T\sum\limits_{i=1}^n q_{t,i}\ln\frac{v_{t+1,i}}{\hat{w}_{t,i}} 
\le 
m(\mathbf{q_{1:T}})\ln\frac{n}{\alpha}+\\
\Big(\sum\limits_{t=2}^T\left\|\mathbf{q_t}\right\|_1-m(\mathbf{q_{1:T}})\Big)
\ln\frac{1}{1-\alpha}
+ \sum\limits_{i=1}^n q_{1,i}\ln\frac{1}{v_{2,i}}.
\end{multline}
\normalsize

Adding the $t=1$ term to the above inequality, we have
\small
\begin{equation}
\begin{array}{ll}
\sum\limits_{t=1}^T\sum\limits_{i=1}^n q_{t,i}\ln\frac{v_{t+1,i}}{\hat{w}_{t,i}} 
\le & \left\|\mathbf{q_1}\right\|_1\ln(n)+m(\mathbf{q_{1:T}})\ln\frac{n}{\alpha}\\
& +\Big(\sum\limits_{t=1}^T\left\|\mathbf{q_t}\right\|_1-m(\mathbf{q_{1:T}})\Big)\ln\frac{1}{1-\alpha}.
\end{array}
\end{equation}
\normalsize

Now we bound the right side, using the choices for $\mathbf{q_t}$
described at the beginning of the proof.  
If $r\ge 2$, $m(\mathbf{q_{1:T}}) = 1$, and $\left\|\mathbf{q_1}\right\|_1 = 0$.
If $r = 1$, $m(\mathbf{q_{1:T}}) = 0$, and $\left\|\mathbf{q_1}\right\|_1 = 1$.
Thus, $m(\mathbf{q_{1:T}}) + \left\|\mathbf{q_1}\right\|_1 = 1$, and the right part can be upper bounded by
$\ln\frac{n}{\alpha}+T\ln\frac{1}{1-\alpha}$.

Combine the above inequality with Eq.(\ref{eq::initial_ineq_entropy}),
set $\mathbf{q_t} = \mathbf{q}\in \mathcal{B}_{\text{n-k}}^\text{n}$ for $t=r,\dots,s$ and $0$ elsewhere,
and multiply both sides by $n-k$,
we have
\begin{equation}
\begin{array}{l}
(1-\exp(-\eta))\sum\limits_{t=r}^s (n-k)\mathbf{w_t}^T\mathbf{\ell_t} - \eta \sum\limits_{t=r}^s (n-k)\mathbf{q}^T\mathbf{\ell_t} \\
\quad \quad \quad \quad \quad \le (n-k)\ln\frac{n}{\alpha}+(n-k)T\ln\frac{1}{1-\alpha}
\end{array}
\end{equation}
If we set $\alpha = 1/(1+(n-k)T)$, then the right part can be upper bounded by $(n-k)\ln(n(1+(n-k)T))+1$,
which equals to $D$ as defined in the Theorem \ref{thm::adaptive_subset_expert}.
Thus, the above inequality can be reformulated as 
\begin{equation}
\sum\limits_{t=r}^s (n-k)\mathbf{w_t}^T\mathbf{\ell_t} \le \frac{\eta \sum\limits_{t=r}^s (n-k)\mathbf{q}^T\mathbf{\ell_t} + D}{1-\exp(-\eta)}
\end{equation}

Since the above inequality holds for arbitrary $\mathbf{q}\in \mathcal{B}_{\text{n-k}}^\text{n}$, we have
\begin{equation}
\label{eq::expert_ineq_before_final}
\sum\limits_{t=r}^s (n-k)\mathbf{w_t}^T\mathbf{\ell_t} \le 
\frac{\eta \min\limits_{\mathbf{q}\in\mathcal{B}_{\text{n-k}}^\text{n}}\sum\limits_{t=r}^s (n-k)\mathbf{q}^T\mathbf{\ell_t} + D}{1-\exp(-\eta)}
\end{equation}

We will apply the inequality in Lemma \ref{lem::freund_ineq} to upper bound the right part in Eq.(\ref{eq::expert_ineq_before_final}).
With $\min\limits_{\mathbf{q}\in\mathcal{B}_{\text{n-k}}^\text{n}}\sum\limits_{t=r}^s (n-k)\mathbf{q}^T\mathbf{\ell_t} \le L$
and $\eta = \ln(1+\sqrt{2D/L})$, we have 
\begin{equation}
\sum\limits_{t=r}^s (n-k)\mathbf{w_t}^T\mathbf{\ell_t} - 
\min\limits_{\mathbf{q}\in\mathcal{B}_{\text{n-k}}^\text{n}}\sum\limits_{t=r}^s (n-k)\mathbf{q}^T\mathbf{\ell_t} \le \sqrt{2LD}+D
\end{equation}

Since the above inequality always holds for all intervals, $[r,s]$,
the result is proved by maximizing the left side over $[r,s]$.
\end{proof}

\section{Proof of Theorem \ref{thm::adaptive_pca}}

\begin{proof}

In the proof, we will use two cases of $Q_t$: $Q_t\in\mathscr{B}_{n-k}^n$, and $Q_t = 0$.

We first apply the eigendecomposition to $Q_t$ as $Q_t = \widetilde{D}\diag(\mathbf{q_t})\widetilde{D}^T$,
where $\widetilde{D} = [\mathbf{\tilde{d}_1},\dots,\mathbf{\tilde{d}_n}]$.
Since in the adaptive setting, $Q_{t-1}$ is either equal to $Q_t$ or $0$,
they share the same eigenvectors and 
$Q_{t-1}$ can be expressed as $Q_{t-1} = \widetilde{D}\diag(\mathbf{q_{t-1}})\widetilde{D}^T$.

According to Lemma \ref{lem::adaptive_pca_step_ineq}, the following inequality is true for both cases of $Q_t$:
\begin{equation}
\label{eq::step_ineq_pca}
\begin{array}{l}
\left\|\mathbf{q_t}\right\|_1\Tr(W_t\mathbf{x_t}\mathbf{x_t}^T)(1-\exp(-\eta)) -\eta\Tr(Q_t\mathbf{x_t}\mathbf{x_t}^T) \\
\quad \quad \quad \quad \le -\Tr(Q_t\ln \widehat{W}_t) + \Tr(Q_t\ln V_{t+1})
\end{array}
\end{equation}

The next steps extend proof of Proposition 2 in \cite{cesa2012new} to the matrix case.

We analyze the right part of the above inequality,
which can be expressed as:
\begin{equation}
\label{eq::quantum_split_two_parts}
\begin{array}{l}
-\Tr(Q_t\ln \widehat{W}_t) + \Tr(Q_t\ln V_{t+1}) = \bar{A} + \bar{B}
\end{array}
\end{equation}
where $\bar{A} = -\Tr(Q_t\ln \widehat{W}_t) + \Tr(Q_{t-1}\ln V_t)$, 
and $\bar{B} = - \Tr(Q_{t-1}\ln V_t) + \Tr(Q_t\ln V_{t+1})$.

We will first upper bound the $\bar{A}$ term, 
and then telescope the $\bar{B}$ term.

$\bar{A}$ can be expressed as:
\begin{equation}
\begin{array}{l}
\bar{A} = \sum\limits_{i:q_{t,i}\ge q_{t-1,i}}\Big(
\underbrace{-\Tr\big((q_{t,i}\mathbf{\tilde{d}_i}\mathbf{\tilde{d}_i}^T-q_{t-1,i}\mathbf{\tilde{d}_i}\mathbf{\tilde{d}_i}^T)
\ln \widehat{W}_t\big)}_{\circled{1}}\\
\quad\quad+ \underbrace{\Tr(q_{t-1,i}\mathbf{\tilde{d}_i}\mathbf{\tilde{d}_i}^T\ln V_t) 
- \Tr(q_{t-1,i}\mathbf{\tilde{d}_i}\mathbf{\tilde{d}_i}^T\ln\widehat{W}_t)}_{\circled{2}}
\Big) \\
\quad\quad+ \sum\limits_{i:q_{t,i}<q_{t-1,i}}\Big(
\underbrace{-\Tr\big((q_{t,i}\mathbf{\tilde{d}_i}\mathbf{\tilde{d}_i}^T
-q_{t-1,i}\mathbf{\tilde{d}_i}\mathbf{\tilde{d}_i}^T)\ln V_t\big)}_{\circled{3}} \\
\quad\quad+ \underbrace{\Tr(q_{t,i}\mathbf{\tilde{d}_i}\mathbf{\tilde{d}_i}^T\ln V_t) 
- \Tr(q_{t,i}\mathbf{\tilde{d}_i}\mathbf{\tilde{d}_i}^T\ln\widehat{W}_t)}_{\circled{4}}
\Big)
\end{array}
\end{equation}

For $\circled{1}$, it can be expressed as:
\begin{equation}
\begin{array}{ll}
\circled{1} &= \Tr\big((q_{t,i}\mathbf{\tilde{d}_i}\mathbf{\tilde{d}_i}^T
- q_{t-1,i}\mathbf{\tilde{d}_i}\mathbf{\tilde{d}_i^T})\ln \widehat{W}_t^{-1}\big) \\
&\le \Tr\big((q_{t,i}\mathbf{\tilde{d}_i}\mathbf{\tilde{d}_i}^T-q_{t-1,i}\mathbf{\tilde{d}_i}\mathbf{\tilde{d}_i}^T)\ln\frac{n}{\alpha}\big) \\
&=(q_{t,i}-q_{t-1,i})\ln\frac{n}{\alpha}.
\end{array}
\end{equation}
The inequality holds because the update in Eq.(\ref{eq::fix_share_pca}) 
implies $\ln\widehat{W}_T^{-1}\preceq I\ln\frac{n}{\alpha}$ and
furthermore,
$(q_{t,i}\mathbf{\tilde{d}_i}\mathbf{\tilde{d}_i}^T-q_{t-1,i}\mathbf{\tilde{d}_i}\mathbf{\tilde{d}_i}^T)$
is positive semi-definite. 
Thus, Lemma \ref{lem::matrix_pos_sym_ineq}, gives the result.

The expression for $\circled{2}$ can be bounded as 
\begin{equation}
\begin{array}{ll}
\circled{2} &= \Tr(q_{t-1,i}\mathbf{\tilde{d}_i}\mathbf{\tilde{d}_i}^T\ln(V_t\widehat{W}_t^{-1})) \\
 & \le q_{t-1,i}\ln\frac{1}{1-\alpha}
\end{array}
\end{equation}
where the equality is due to the fact that $V_t$ and $\widehat{W}_t$
have the same eigenvectors.
The inequality follows since $\ln(V_t\widehat{W}_t^{-1})\preceq
I\ln\frac{1}{1-\alpha}$, due to the update in
Eq.(\ref{eq::fix_share_pca}),
while $q_{t-1,i}\mathbf{\tilde{d}_i}\mathbf{\tilde{d}_i}^T$ is
positive semi-definite. Thus Lemma \ref{lem::matrix_pos_sym_ineq} gives the result.

The bound $\circled{3}$ can be expressed as:
\begin{equation}
\circled{3} = \Tr\big((-q_{t,i}\mathbf{\tilde{d}_i}\mathbf{\tilde{d}_i}^T
+ q_{t-1,i}\mathbf{\tilde{d}_i}\mathbf{\tilde{d}_i}^T)\ln V_t\big) \le 0
\end{equation}
Here, the inequality follows since
$\ln V_t \preceq 0$ and 
and
$(-q_{t,i}\mathbf{\tilde{d}_i}\mathbf{\tilde{d}_i}^T+q_{t-1,i}\mathbf{\tilde{d}_i}\mathbf{\tilde{d}_i}^T)$
is positive semi-definite.
Thus, Lemma \ref{lem::matrix_pos_sym_ineq} gives the result.

For $\circled{4}$, we have $\circled{4} \le q_{t,i}\ln\frac{1}{1-\alpha}$,
which follows the same argument used to bound the term $\circled{2}$.

Thus, $\bar{A}$ can be upper bounded as follows:
\small
\begin{equation}
\begin{array}{ll}
\bar{A} \le & \underbrace{\sum\limits_{i:q_{t,i}\ge q_{t-1,i}} (q_{t,i}-q_{t-1,i})}_{= D_{TV}(\mathbf{q_t},\mathbf{q_{t-1}})}\ln\frac{n}{\alpha} \\
& + \underbrace{\Big(\sum\limits_{i:q_{t,i}\ge q_{t-1,i}}q_{t-1,i}
+ \sum\limits_{i:q_{t,i}<q_{t-1,i}} q_{t,i}\Big)}_{=\left\|\mathbf{q_t}\right\|_1-D_{TV}(\mathbf{q_t},\mathbf{q_{t-1}})}
\ln\frac{1}{1-\alpha} 
\end{array}
\end{equation}
\normalsize

Then we telescope the $\bar{B}$ term, substitute the above inequality for $\bar{A}$ into Eq.(\ref{eq::quantum_split_two_parts}), 
and sum over $t=2,\dots,T$ to give:
\scriptsize
\begin{equation}
\begin{array}{l}
\sum\limits_{t=2}^T \Big(-\Tr(Q_t\ln \widehat{W}_t) + \Tr(Q_t\ln V_{t+1})\Big) \\
\le m(\mathbf{q_{1:T}})\ln\frac{n}{\alpha}+\Big(\sum\limits_{t=2}^T\left\|\mathbf{q_t}\right\|_1-m(\mathbf{q_{1:T}})\Big)
\ln\frac{1}{1-\alpha}
- \Tr(Q_1\ln V_2)
\end{array}
\end{equation}
\normalsize

Adding the $t=1$ term to the above inequality, we have 
\scriptsize
\begin{equation}
\begin{array}{l}
\sum\limits_{t=1}^T \Big(-\Tr(Q_t\ln \widehat{W}_t) + \Tr(Q_t\ln V_{t+1})\Big) \\
\le \left\|\mathbf{q_1}\right\|_1\ln(n)+m(\mathbf{q_{1:T}})\ln\frac{n}{\alpha}
 +\Big(\sum\limits_{t=1}^T\left\|\mathbf{q_t}\right\|_1-m(\mathbf{q_{1:T}})\Big)\ln\frac{1}{1-\alpha}
\end{array}
\end{equation}
\normalsize

For the above inequality, we set $Q_t = Q\in \mathscr{B}_{n-k}^n$ for $t=r,\dots,s$ and $0$ elsewhere,
which makes $\mathbf{q_t} = \mathbf{q}\in \mathcal{B}_{n-k}^n$ for $t=r,\dots,s$ and $0$ elsewhere.
If $r\ge 2$, $m(\mathbf{q_{1:T}}) = 1$, and $\left\|\mathbf{q_1}\right\|_1 = 0$.
If $r = 1$, $m(\mathbf{q_{1:T}}) = 0$, and $\left\|\mathbf{q_1}\right\|_1 = 1$.
Thus, $m(\mathbf{q_{1:T}}) + \left\|\mathbf{q_1}\right\|_1 = 1$, and the right part can be upper bounded by
$\ln\frac{n}{\alpha}+T\ln\frac{1}{1-\alpha}$.

The rest of the steps follow exactly the same as in the proof of Theorem \ref{thm::adaptive_subset_expert}.
\end{proof}

\section{Proof of Lemma \ref{lem::mat_lem1_bianchi}}

\begin{proof}
We first deal with the term $\Tr(Q_t\ln V_{t+1})$.
According to the update in Eq.(\ref{eq::var_unit_v_t+1}), we have
\begin{equation}
\begin{array}{l}
\Tr(Q_t\ln V_{t+1}) = \Tr\bigg(Q_t\ln\Big(\frac{\exp(\ln Y_t - \eta C_t)}{\Tr(\exp(\ln Y_t - \eta C_t))}\Big)\bigg) \\
 = \Tr\big( Q_t(\ln Y_t -\eta C_t)\big) - \ln \Big(\Tr\big(\exp(\ln Y_t  - \eta C_t)\big)\Big),
\end{array}
\end{equation}
since $Q_t \in\mathscr{B}_1^n$ and $\Tr(Q_t) = 1$.

As a result, we have
$\Tr(Q_t\ln V_{t+1}) - \Tr(Q_t\ln Y_t)$ $=$ $-\eta\Tr(Q_tC_t) - \ln \Big(\Tr\big(\exp(\ln Y_t  - \eta C_t)\big)\Big)$.

Thus, to prove the inequality in Lemma \ref{lem::mat_lem1_bianchi},
it is enough to prove the following inequality
\begin{equation}
\eta\Tr(Y_tC_t) - \frac{\eta^2}{2} + \ln \Big(\Tr\big(\exp(\ln Y_t  - \eta C_t)\big)\Big) \le 0
\end{equation} 

Before we proceed, we need the following lemmas:
\begin{lemma}[Golden-Thompson inequality]
For any symmetric matrices $A$ and $B$, the following inequality holds:
\begin{equation*}
\Tr\big(\exp(A + B)\big) \le \Tr\big(\exp(A)\exp(B)\big)
\end{equation*}
\end{lemma}

\begin{lemma}[Lemma 2.1 in \cite{tsuda2005matrix}]
\label{lem::lem21_tsuda}
For any symmetric matrix $A$ such that $0\preceq A\preceq I$ and any $\rho_1,\rho_2 \in\mathbb{R}$,
the following holds:
\begin{equation*}
\exp\big(A\rho_1 + (I-A)\rho_2\big) \preceq A\exp(\rho_1) + (I-A)\exp(\rho_2)
\end{equation*}
\end{lemma}


Then we apply the Golden-Thompson inequality to the term $\Tr\big(\exp(\ln Y_t  - \eta C_t)\big)$,
which gives us the inequality below:
\begin{equation}
\Tr\big(\exp(\ln Y_t  - \eta C_t)\big) \le \Tr(Y_t\exp(-\eta C_t)).
\end{equation}
For the term $\exp(-\eta C_t)$, 
by applying the Lemma \ref{lem::lem21_tsuda} with $\rho_1 = -\eta$ and $\rho_2 = 0$,
we will have the following inequality:
\begin{equation}
 \exp(-\eta C_t) \preceq I - C_t(1-\exp(-\eta)).
\end{equation}
Thus, we will have
\begin{equation}
\Tr(Y_t\exp(-\eta C_t)) \le 1 - \Tr(Y_tC_t)(1-\exp(-\eta)),
\end{equation}
and 
\begin{equation}
\Tr\big(\exp(\ln Y_t  - \eta C_t)\big) \le 1 - \Tr(Y_tC_t)(1-\exp(-\eta)),
\end{equation}
since $Y_t \in\mathscr{B}_1^n$ and $\Tr(Y_t) = 1$.

Thus, it is enough to prove the following inequality
\begin{equation}
\eta\Tr(Y_tC_t) - \frac{\eta^2}{2} + \ln \Big(1 - \Tr(Y_tC_t)(1-\exp(-\eta))\Big) \le 0
\end{equation} 

Since $\ln(1-x) \le -x$, we have
\small
\begin{equation}
\ln \Big(1 - \Tr(Y_tC_t)(1-\exp(-\eta))\Big) \le  -\Tr(Y_tC_t)(1-\exp(-\eta)).
\end{equation}
\normalsize
Thus, it suffices to prove the following inequality:
\begin{equation} 
\big(\eta-1+\exp(-\eta)\big)\Tr(Y_tC_t) - \frac{\eta^2}{2} \le 0
\end{equation}

Note that by using convexity
of $\exp(-\eta)$, $\eta-1+\exp(-\eta) \ge 0$.

By applying Lemma \ref{lem::matrix_pos_sym_ineq} with $A = Y_t$, $B = C_t$, and $C = I$,
we have $\Tr(Y_tC_t) \le \Tr(Y_t) = 1$.
Thus, when $\eta \ge 0$, it is enough to prove the following inequality 
\begin{equation}
\eta-1+\exp(-\eta) - \frac{\eta^2}{2} \le 0.
\end{equation}
This inequality follows from convexity of
$\frac{\eta^2}{2}-\exp(-\eta)$ over $\eta\ge 0$. 
\end{proof}

\section{Proof of Theorem \ref{thm::adaptive_var_simplex}}

\begin{proof}

First, since $0\preceq C_t\preceq I$, we have $\max_{i,j}|C_t(i,j)|$ $\le 1$.

Before we proceed, we need the following lemma from \cite{warmuth2006online}
\begin{lemma}[Lemma 1 in \cite{warmuth2006online}]
\label{lem::var_simplex_step}
Let $\max_{i,j}|C_t(i,j)|\le\frac{r}{2}$, then
for any $\mathbf{u_t}\in\mathcal{B}_1^n$, any constants $a$ and $b$ such that $0\le a \le \frac{b}{1+rb}$,
and $\eta = \frac{2b}{1+rb}$, we have
\begin{equation*}
a\mathbf{y_t}^TC_t\mathbf{y_t} - b\mathbf{u_t}^TC_t\mathbf{u_t} \le d(\mathbf{u_t},\mathbf{y_t}) - d(\mathbf{u_t},\mathbf{v_{t+1}})
\end{equation*}
\end{lemma}

Now we apply Lemma \ref{lem::var_simplex_step} 
under the conditions  $r =2$, $a = \frac{b}{2b+1}$, $\eta = 2a$, 
and $b = \frac{c}{2}$.

Recall that $d(\mathbf{u_t},\mathbf{y_t}) - d(\mathbf{u_t},\mathbf{v_{t+1}})$ $=$ $\sum_{i}u_{t,i}\ln\Big(\frac{v_{t+1,i}}{y_{t,i}}\Big)$.
Combining this with the inequality in Lemma \ref{lem::var_simplex_step} 
and the fact that $\left\|\mathbf{u_t}\right\|_1 = 1$,
we have 
\begin{equation}
a\left\|\mathbf{u_t}\right\|_1\mathbf{y_t}^TC_t\mathbf{y_t} - b\mathbf{u_t}^TC_t\mathbf{u_t} \le \sum_{i}u_{t,i}\ln\Big(\frac{v_{t+1,i}}{y_{t,i}}\Big)
\end{equation}

Note that the above inequality is also true when $\mathbf{u_t} = 0$.

Note that the right side of the above inequality 
 is the same as the right part of the Eq.(\ref{eq::initial_ineq_entropy}) 
in the proof of Theorem \ref{thm::adaptive_subset_expert}.

As a result, we will use the same steps as in the proof of Theorem \ref{thm::adaptive_subset_expert}.
Then we will set $\mathbf{u_t} = \mathbf{u} = \argmin_{\mathbf{q}\in\mathcal{B}_1^n}\sum\limits_{t=r}^s \mathbf{q}^TC_t\mathbf{q}$
for $t = r,\dots,s$, and $0$ elsewhere. Summing from $t=1$ up to $T$,
gives the following inequality:
\begin{equation}
a\big[ \sum_{t=r}^s \mathbf{y_t}^TC_t\mathbf{y_t} \big] - b\big[\min_{\mathbf{u}\in\mathcal{B}_1^n}\sum_{t=r}^s\mathbf{u}^TC_t\mathbf{u}\big] 
\le \ln\frac{n}{\alpha} + T\ln\frac{1}{1-\alpha}
\end{equation}

Since $\alpha = 1/(T+1)$, $T\ln\frac{1}{1-\alpha} \le 1$.
Then the above inequality becomes
\begin{equation}
a\big[ \sum\limits_{t=r}^s \mathbf{y_t}^TC_t\mathbf{y_t} \big] 
- b\big[\min_{\mathbf{u}\in\mathcal{B}_1^n}\sum\limits_{t=r}^s\mathbf{u}^TC_t\mathbf{u}\big] 
\le \ln\big((1+T)n\big) + 1
\end{equation}

Plugging in the expressions of $a = c/(2c+2)$, $b = c/2$, and $c = \frac{\sqrt{2\ln\big((1+T)n\big)+2}}{\sqrt{L}}$ we will have
\begin{equation}
\begin{array}{l}
\sum\limits_{t=r}^s \mathbf{y_t}^TC_t\mathbf{y_t} - \min_{\mathbf{u}\in\mathcal{B}_1^n}\sum\limits_{t=r}^s\mathbf{u}^TC_t\mathbf{u} \\
\le c\Big[\min_{\mathbf{u}\in\mathcal{B}_1^n}\mathbf{u}^TC_t\mathbf{u}\Big] + 2\frac{c+1}{c}\big(\ln\big((1+T)n\big) + 1\big) \\
\le cL + 2\frac{c+1}{c}\big(\ln\big((1+T)n\big) + 1\big) \\
= 2\sqrt{2L\Big(\ln\big((1+T)n\big)+1\Big)} + 2\ln\big((1+T)n\big)
\end{array}
\end{equation}

Since the inequality holds for any $1\le r\le s \le T$, the proof is
concluded by maximizing over $[r,s]$ on the left.
\end{proof}


\end{document}